\def\year{2018}\relax
\pdfoutput=1
\documentclass[letterpaper]{article} %DO NOT CHANGE THIS
\usepackage{aaai18}  %Required
\usepackage{times}  %Required
\usepackage{helvet}  %Required
\usepackage{courier}  %Required
\usepackage{url}  %Required
\usepackage{graphicx}  %Required
\usepackage[ruled,lined]{algorithm2e}

\usepackage{amsmath,amssymb,amsthm}

\newtheorem{theorem}{Theorem}
\newtheorem{corollary}{Corollary}[theorem]
\newtheorem{lemma}[theorem]{Lemma}
\newtheorem{proposition}[theorem]{Proposition}
\newtheorem{definition}[]{Definition}

\newtheorem{remark}[]{Remark}

\newcommand{\set}[1]{\mathcal{\uppercase{#1}}}

\newcommand{\op}[1]{#1}
\newcommand{\R}{\mathbb{R}}
\newcommand{\Z}{\mathbb{Z}}

\DeclareMathOperator{\E}{\mathbb{E}}

\DeclareMathOperator*{\limk}{\lim_{k \rightarrow \infty}}

\begin{document}

\title{Unifying Value Iteration, Advantage Learning, and Dynamic Policy Programming}
\author{Tadashi Kozuno\textsuperscript{1}\\
\texttt{tadashi.kozuno@oist.jp}
\And
Eiji Uchibe\textsuperscript{1, 2}\\
\texttt{uchibe@atr.jp}
\And
Kenji Doya\textsuperscript{1}\\
\texttt{doya@oist.jp}
\AND
\\
\textsuperscript{1}Neural Computation Unit, Okinawa Institute of Science and Technology\\
\textsuperscript{2}Department of Brain Robot Interface, ATR Computational Neuroscience Laboratories}

\maketitle

\begin{abstract}
    Approximate dynamic programming algorithms, such as approximate value iteration, have been successfully applied to many complex reinforcement learning tasks, and a better approximate dynamic programming algorithm is expected to further extend the applicability of reinforcement learning to various tasks. In this paper we propose a new, robust dynamic programming algorithm that unifies value iteration, advantage learning, and dynamic policy programming. We call it generalized value iteration (GVI) and its approximated version, approximate GVI (AGVI). We show AGVI's performance guarantee, which includes performance guarantees for existing algorithms, as special cases. We discuss theoretical weaknesses of existing algorithms, and explain the advantages of AGVI. Numerical experiments in a simple environment support theoretical arguments, and suggest that AGVI is a promising alternative to previous algorithms.
\end{abstract}

\section{Introduction}
Approximate dynamic programming (approximate DP or ADP) approximates each iteration of DP in two ways: estimating the Bellman operator using empirical samples and/or expressing a Q-value function by a function approximator. Many reinforcement learning (RL) algorithms are based on ADP. For example, Q-learning is an instance of approximate value iteration (approximate VI or AVI). Recently, a combination of deep learning and AVI is becoming increasingly popular because its performance has exceeded that of human experts in many Atari games \cite{mnih-dqn-2015,hasseltDoubleDQN}.

However, theoretical analysis of AVI shows that even when approximation errors are i.i.d.~Gaussian noise, AVI may not be able to find an optimal policy. Unfortunately, approximate policy iteration (approximate PI or API) has almost the same performance guarantee \cite{ndp,scherrer2012_icml}, and a better ADP algorithm is necessary to further extend the applicability of RL to complex problems.

Recently, value-based algorithms using new DP operators have been proposed by several researchers \cite{dpp,increaseActionGap}. Bellemare et al. showed that a class of operators including an advantage learning (AL) operator can be used to find an optimal policy when there is no approximation error \cite{increaseActionGap}. In particular, it was shown experimentally that deep RL based on approximate AL (AAL) outperforms an AVI-based deep RL algorithm called deep Q-network (DQN). However, AAL lacks a performance guarantee.

Azar et al. proposed dynamic policy programming (DPP) and approximate DPP (ADPP). The latter displays greater robustness to approximation errors than AVI \cite{dpp}. In particular, if cumulative approximation errors over iterations is $0$, ADPP finds an optimal policy. However, despite its theoretical guarantee, ADPP has been rarely used for complex tasks, with a few exceptions \cite{Tsurumine2017a}.

\begin{figure}[t]
  \includegraphics[width=\linewidth]{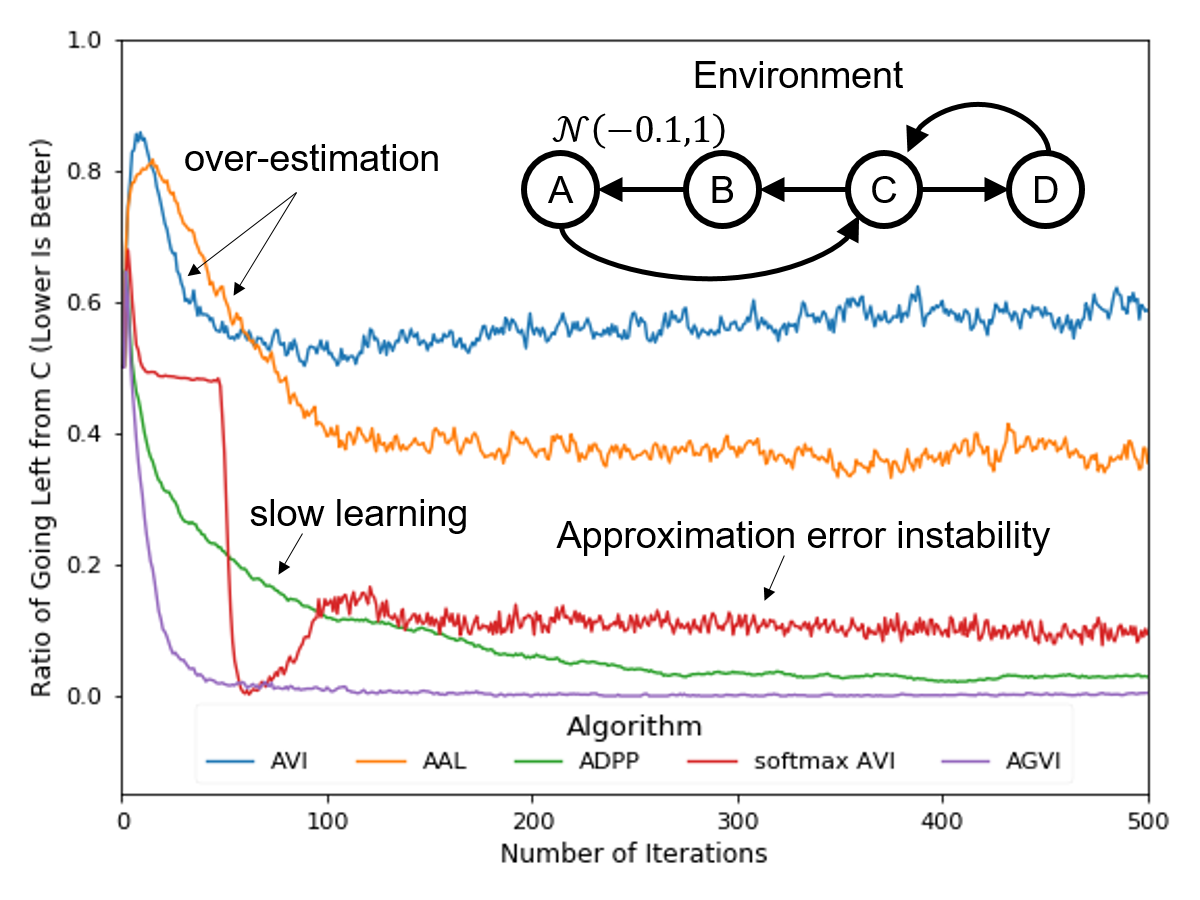}
  \caption{Experimental results in a toy-problem shown in the inset (a modified version of an environment of Fig~6.8 in \cite{sutton}). Lines indicate the mean ratio (over $1,000$ experimental runs) of going left from C in an evaluation phase where greedy action is taken. At each state, there are $100$ possible actions. At state $C$, actions $1-50$ lead to the left, but the rest lead to the right. At states $A$ and $D$, any action leads to state $C$. At state $B$, any action leads to state $A$ with an immediate reward drawn from a Gaussian distribution, the mean of which is $-0.1$, and variance is $1.0$. The agent is allowed to take actions according to $\varepsilon$-greedy where $\varepsilon = 0.1$ for $5$ steps. Every $5$ steps, the agent is initialized to state $C$, and the Q-table is updated using experience the agent obtained in the $5$ steps. AVI and AAL suffer from maximization bias, and ADPP suffers from slow learning. Softmax AVI, where mellowmax is used rather than max, learns quickly without maximization bias, but suffers from error instability. Our algorithm, AGVI, avoids those issues.} \label{fig:toy_problem}
\end{figure}

Motivated by those studies, we propose a new DP algorithm called generalized VI (GVI), which unifies VI, DPP, and AL. We provide a performance guarantee of approximate GVI (AGVI), which not only shows that the price to pay for ADPP's robustness to approximation errors is a prolonged effect of approximation errors, but also provides a performance guarantee for AAL. Furthermore, we argue that AAL tends to over-estimate Q-values by maximization bias \cite{hasseltDoubleQ} as the cost of its optimality. AGVI provides a way to balance the pros and cons of other algorithms, leading to better performance as exemplified in Fig.~\ref{fig:toy_problem}.

We also explain how GVI is related to a regularized policy search method, in which a policy is updated repeatedly with constraints on new and old policies. We show a relationship between the Q-value function learned by GVI and regularization coefficients. Such a connection has not been demonstrated for AL.

Finally, we show experimental results for AGVI in simple environments. The results support our theoretical argument, and suggest that AGVI is a promising alternative.

In summary, here we:
\begin{itemize}
    \item propose a new DP algorithm called GVI the approximated version of which is robust to approximation errors and resistant to over-estimation of Q-values.
    \item show AGVI's performance guarantee, which indicates a weakness of ADPP.
    \item show a performance guarantee for ALL.
    \item clarify a connection with existing DP algorithms and a regularized policy search method.
\end{itemize}

\section{Preliminaries}
\subsection{Basic Definitions and Notations}
We denote the set of bounded real-valued functions over a finite set $\set{X}$ by $\set{B}_{\set{X}}$. We often consider a Banach space $\left( \set{B}_{\set{X}}, \| \cdot \| \right)$, where $\| f \| = \max_{x \in \set{X}} |f(x)|$. For brevity, we denote it by $\set{B}_{\set{X}}$ with an abuse of notation. When we say a series of functions $f_k$ converges to $f$, we mean uniform converges, and we write $\limk f_k = f$.

For functions $f$ and $g$ with a domain $\set{X}$, $f > (\geq) g$ mean $f(x) > (\geq) g(x)$ for any $x \in \set{X}$. Similarly, any arithmetic operation of two functions is a component-wise operation. For example, $f + g$ is a function $f(x) + g(x)$, and $f + c$ is a function $f(x) + c$ for any constant $c$.

\subsection{Reinforcement Learning}
We only consider the following type of Markovian decision processes (MDPs):
\begin{definition}[Finite State and Action MDP]
An MDP is a 5-tuple of $\left( \set{S}, \set{A}, P, r, \gamma \right)$, where $\set{S}$ is the finite state space, $\set{A}$ is the finite action space, $P: \set{S} \times \set{A} \times \set{S} \rightarrow [0, 1]$ is the state transition probability kernel, $r: \set{S} \times \set{A} \rightarrow [-r_{max}, r_{max}]$ is the expected immediate reward function, and $\gamma \in [0, 1)$ is the discount factor.
\end{definition}
Semantics are as follows: suppose that an agent has executed an action $a \in \set{A}$ at a state $s \in \set{S}$. Then, state transition to a subsequent state $s' \sim P(\cdot|s, a)$ occurs with an immediate reward whose expected value is $r(s, a)$. We usually use $s$ and $a$ to denote a state and an action, respectively. We only consider infinite horizon tasks.

A policy is a conditional probability distribution over actions given a state. We consider only stationary stochastic Markov policies.

The state value function (for a policy $\pi$) is the expected discounted future rewards when the policy $\pi$ is followed from a state $s$, in other words, $V^{\pi} (s) := \E^\pi [ \sum_{t \geq 0} \gamma^t r_t \mid s_0 = s ]$, where $\E^\pi$ indicates that a policy $\pi$ is followed with the expectation, and $r_t$ and $s_t$ denote reward and state at time $t$, respectively. When the expectation is further conditioned by an action $a$, it is called a Q-value function. We denote it as $Q^{\pi} (s, a) := \E^\pi [ \sum_{t \geq 0} \gamma^t r_t \mid s_0 = s, a_0=a ]$. It is known that $V^* := \sup_\pi V^{\pi}$ and $Q^* := \sup_\pi Q^{\pi}$ exist under our settings, and they are called the optimal state value function and the optimal Q-value function, respectively. An optimal policy $\pi^*$ satisfies $V^* = V^{\pi^*}$. The optimal advantage function is defined as $A^* := Q^* - V^*$.

\subsection{Bellman Operator and Policy Operators}
An operator is a mapping between functional spaces. A policy $\pi$ yields a right-linear operator $\set{B}_{\set{S} \times \set{A}} \rightarrow \set{B}_{\set{S}}$ defined by $\forall f \in \set{B}_{\set{S} \times \set{A}}$, $(\pi f) (s) = \sum_a \pi(a|s) f(s, a)$. A stochastic kernel $P$ also yields a right-linear operator $\set{B}_{\set{S}} \rightarrow \set{B}_{\set{S} \times \set{A}}$ defined as $(P g) (s, a) = \sum_{s'} P(s'|s, a) g(s')$, where $g \in \set{B}_{\set{S}}$. By combining them, we define the following right-linear operator $(P^{\pi} f) (s, a) := (P (\pi f)) (s, a)$. Hereafter, we omit parentheses, e.g., $(P^\pi f)$, and denote it as $P^\pi f$ for brevity.

The Bellman operator $\op{T}^\pi: \set{B}_{\set{S} \times \set{A}} \rightarrow \set{B}_{\set{S} \times \set{A}}$ for a policy $\pi$ is defined s.t.~$\forall f \in \set{B}_{\set{S} \times \set{A}}$, $\op{T}^\pi f = r + \gamma P^{\pi} f$. Similarly, the Bellman optimality operator $\op{T}$ is defined s.t.~$\forall f \in B_{\set{S} \times \set{A}}$, $\op{T} f = r + \gamma P \op{m} f$, where $\op{m}$ is max operator defined by $\op{m} f (s) = \max_a f(s, a)$. We often use the mellowmax operator $\op{m}_\beta: \set{B}_{\set{S} \times \set{A}} \rightarrow \set{B}_{\set{S}}$ defined by
\begin{align*}
    \op{m}_\beta f (s) := \frac{1}{\beta} \log \sum_a \frac{\exp \left( \beta f(s, a) \right)}{|\set{A}|} ,
\end{align*}
where $|\set{A}|$ is the number of actions \cite{asadi_mellowmax}. It is known that as $\beta \rightarrow \infty$, $\op{m}_\beta \rightarrow \op{m}$. On the other hand, $\lim_{\beta \rightarrow 0} m_\beta$ becomes just an average over actions. Therefore, by $\op{m}_\infty$ and $\op{m}_0$, we mean $\op{m}$ and an average over actions, respectively, in this paper. We define $\op{T}_\beta: \set{B}_{\set{S} \times \set{A}} \rightarrow \set{B}_{\set{S} \times \set{A}}$ s.t.~$\forall f \in \set{B}_{\set{S} \times \set{A}}$, $\op{T}_\beta f = r + \gamma P \op{m}_\beta f$. Mellowmax is known to be a non-expansion \cite{asadi_mellowmax}. Therefore, $\op{T}_\beta$ is a contraction with modulus $\gamma$. We denote its unique fixed point by $Q^\beta$.

The following operator is often used in RL:
\begin{equation*}
    \op{b}_\beta f (s) = \sum_a \frac{\exp \left( \beta f (s, a) \right) f(s, a)}{\sum_{a'} \exp \left( \beta f (s, a') \right)}.
\end{equation*}
We call $\op{b}_\beta$ the Boltzmann operator, which is not a non-expansion \cite{asadi_mellowmax}.

\subsection{Advantage Learning Operator}
Bellemare et al. \cite{increaseActionGap} proposed an AL operator:
\begin{align}
    Q_{k+1} := \op{T} Q_k + \alpha \left(  Q_k - \op{m} Q_k \right),\label{eq:advantage_learning}
\end{align}
where $\alpha \in [0, 1)$. The algorithm using this update rule is called AL. They showed that a greedy policy w.r.t. $\limk Q_k$ is an optimal policy when there is no approximation error. Furthermore, Bellmare et al. argued that by using AL, the difference between Q-values for an optimal action and for sub-optimal actions is enhanced, leading to learning that is less susceptible to function approximation error. They experimentally showed that deep RL based on AAL outperforms DQN in Atari games.

\subsection{Dynamic Policy Programming Operator}\label{subsec:dpp_operator}
Azar et al. \cite{dpp} proposed the following update rule called DPP:
\begin{align}\label{eq:dpp_mellowmax}
    Q_{k+1} := \op{T}_\beta Q_{k} + Q_{k} - \op{m}_\beta Q_{k},
\end{align}
where $\beta \in (0, \infty]$. Since the difference between $\op{m}_\beta Q_k$ and $\op{b}_\beta Q_k$ can be bounded, they also proposed the following update rule:
\begin{align*}
    Q_{k+1} := r + \gamma P \op{b}_\beta Q_k + Q_{k} - \op{b}_\beta Q_{k}.
\end{align*}
They showed that a Boltzmann action selection policy $\pi_k(a|s) = \exp \left( \beta Q_k (s, a) \right) / \sum_{a'} \exp \left( \beta Q_k (s, a') \right)$ converges to an optimal policy, and that ADPP is more robust to approximation errors than AVI or API.

\section{The Algorithm and Theoretical Analyses}
\subsection{Generalized Value Iteration (GVI)}
Note that r.h.s. of (\ref{eq:dpp_mellowmax}) becomes an AL operator with $\alpha=1$ as $\beta \rightarrow \infty$. Consequently, one may think that $Q_{k+1} := \op{T}_\beta Q_{k} + \alpha \left( Q_k - \op{m}_\beta Q_k \right)$ also converges to the optimal Q-value function. Unfortunately, it does not hold. Specifically, the following theorem holds (All proof is in Appendix).
\begin{theorem}[Generalized Value Iteration]\label{theorem:generalizedVI}
Suppose a function $Q_0 \in \set{B}_{\set{S} \times \set{A}}$ and the following update rule
\begin{equation}\label{eq:gvi_update}
    Q_{k+1} := \op{T}_\beta Q_{k} + \alpha \left( Q_k - \op{m}_\beta Q_k \right),
\end{equation}
where $\alpha \in [0, 1]$, $\beta \in (0, \infty]$. If $\alpha \neq 1$,
\begin{align*}
    \limk Q_k =
        \begin{cases}
            \op{m}_{\theta} Q^\theta + \frac{1}{1-\alpha} \left( Q^\theta - \op{m}_{\theta} Q^\theta \right) & \mbox{ if $\beta \neq \infty$}\\
            V^* + \frac{1}{1-\alpha} \left( Q^* - V^* \right) & \mbox{ if $\beta = \infty$}
        \end{cases},
\end{align*}
where $\theta = \frac{\beta}{1-\alpha}$. If $\alpha = 1$, $\forall \varepsilon \in \R^+$, $\exists K \in \Z^+$ s.t. $\forall k > K$, $\exists \phi \in \set{B}_{\set{S} \times \set{A}}, \| \phi \| < \varepsilon$ s.t. \begin{align*}
    Q_k = V^* + Q_0 + k A^* - \op{m}_\beta \left( (k-1)A^* + Q_0 \right) + \phi.
\end{align*}
\end{theorem}
\noindent We call the algorithm using the update rule (\ref{eq:gvi_update}) GVI.

\begin{remark}
Theorem~\ref{theorem:generalizedVI} for $\alpha = 1$ states that for any required accuracy $\varepsilon$, there exists $K$ s.t.~for $k > K$, the deviation of $Q_k$ from $V^* + Q_0 + k A^* - \op{m}_\beta \left( (k-1)A^* + Q_0 \right)$ is kept within $\varepsilon$. Note that unless $A^* = 0$, $Q_k$ does not converge. Hence, we cannot state it in a form similar to cases where $\alpha \neq 1$. Clearly, when $A^*(s, a) < 0$ for a state $s$ and an action $a$, $Q_k (s, a)$ is diverging to $-\infty$. Accordingly, it follows that a greedy policy w.r.t.~$\limk Q_k$ is an optimal. When $A^* = 0$, any action is optimal.
\end{remark}

\begin{remark}
Let us note that when the greedy policy w.r.t. $Q_{\theta}$ is optimal. (Hence, the greedy policy w.r.t. $\limk Q_k$ is optimal) Suppose that an optimal action $a$ and a second-optimal action $b$ satisfies $Q^*(s, a) - Q^*(s, b) > \frac{\gamma \log |\set{A}|}{1-\gamma}\frac{1-\alpha}{\beta}$. Then, $Q_{\theta} (s, a) \geq Q^*(s, a) - \frac{\gamma \log |\set{A}|}{1-\gamma}\frac{1-\alpha}{\beta} > Q^*(s, b) \geq Q_{\theta} (s, b)$. Hence, $a$ is also greedy action w.r.t. $Q_{\theta}$. It implies that whether or not a greedy action w.r.t. $\limk Q_k (s, a)$ is optimal depends on $Q^*(s, a) - Q^*(s, b)$, which is called action-gap \cite{farahmand_action_gap}. When action-gap is large, a task is easy, and GVI may find an optimal policy. On the other hand, when action-gap is small, a task is difficult, and GVI is likely not to find an optimal policy. However, a second-optimal action also has a Q-value close to the best action. Accordingly, the second-optimal action may not be a bad choice.
\end{remark}

As clearly shown by Theorem~\ref{theorem:generalizedVI}, unless either $\alpha = 1$ or $\beta = \infty$, an optimal policy cannot be obtained by GVI. However, empirical results show that both AAL and ADPP work best when $\alpha$ and $\beta$ take moderate values rather than $0$ or $1$ for $\alpha$ and $\infty$ for $\beta$ \cite{dpp,increaseActionGap}. Our theoretical analyses (Theorem~\ref{theorem:bound_agvi} and Sect.~\ref{subsec:theoretical_analysis}) indeed indicate that moderate values of $\alpha$ and $\beta$ have preferable properties.

\subsection{Performance Bound for Approximate GVI}\label{subsec:theoretical_analysis}
An exact implementation of GVI requires a model of an environment. In model-free RL, sampling by a behavior policy introduces sampling errors and bias on chosen state-action pairs. In addition, in large scale problems, function approximation is inevitable. As a result, GVI updates are contaminated with approximation errors $\varepsilon_k \in \set{B}_{\set{S} \times \set{A}}$ resulting in an update rule $Q_{k+1} := \op{T}_\beta Q_{k} + \alpha \left( Q_k - \op{m}_\beta Q_k \right) + \varepsilon_k$. We call this algorithm approximate GVI (AGVI). The following theorem relates approximation error $\varepsilon_k$ and the quality of a policy obtained by AGVI.

\begin{theorem}[Performance Bound for AGVI]\label{theorem:bound_agvi}
Suppose the update rule of AGVI, and $Q_0 \in \set{B}_{\set{S} \times \set{A}}$ s.t.~$\| Q_0 \| \leq \frac{r_{max}}{1-\gamma} := V_{max}$. Furthermore, let $\pi_k$ denote a policy which satisfies $\pi_k Q_k = \op{m}_\beta Q_k$. Then, we have
\begin{align}\label{eq:bound_agvi}
    \left\| Q^* - Q^{\pi_k} \right\|
    \leq  C + \frac{2}{ 1 - \gamma }\frac{ 1-\alpha}{ 1-\alpha^{k+1}} \left( C_k + \mathcal{E}_k \right),
\end{align}
where
\begin{align*}
    C &:=  \frac{\gamma}{1-\gamma} \frac{1-\alpha}{\beta} \log |\set{A}|,\\
    C_k &:= \gamma \frac{\alpha^{k+1} - \gamma^{k+1}}{\alpha - \gamma} \left( 2 V_{max} + \frac{\alpha}{\beta} \log |\set{A}| \right),\\
    \mathcal{E}_k &:= \sum_{i=0}^{k} \gamma^i \left\| \sum_{j=0}^{k-i} \alpha^j \varepsilon_{k-i-j} \right\|.
\end{align*}
\end{theorem}

\begin{remark}
$\pi_k$ which satisfies $\pi_k Q_k = \op{m}_\beta Q_k$ can be found, for example, by maximizing the entropy of $\pi_k$ with constraints as proposed in \cite{asadi_mellowmax}.
\end{remark}

\begin{remark}
As $\alpha$ approaches $1$, this performance bound reconstructs that of ADPP \cite{dpp}\footnote{We corrected a mistake in their bound (their error terms lack a coefficient 2).}:
\begin{align*}
    \left\| Q^* - Q^{\pi_k} \right\|
    \leq \frac{2}{ (1 - \gamma) (k + 1) } \left( C_k + \mathcal{E}_k \right),
\end{align*}
where $\alpha$ in $C_k$ and $\mathcal{E}_k$ is set to $1$. As a corollary, a performance bound for AAL can be obtained by $\beta \rightarrow \infty$ as well.
\end{remark}

\subsubsection{Faster Error Decay with $\alpha < 1$}
Theorem~\ref{theorem:bound_agvi} implies a slow decay of approximation error when $\alpha = 1$. For simplicity, assume that $\varepsilon_k = 0$ for all $k$ except $0$ where $\varepsilon_0 = \varepsilon$. In this case, (\ref{eq:bound_agvi}) becomes
\begin{align}
    \left\| Q^* - Q^{\pi_k} \right\| \leq C + \frac{2}{1-\gamma} D_{k} \varepsilon,
\end{align}
where $D_k := \gamma^k \left( \sum_{i=0}^k \alpha^i \gamma^{-i} \right) \left( \sum_{i=0}^k \alpha^i \right)^{-1}$, and all terms not related to approximation error are aggregated to $C$. Therefore, $D_k$ determines how quickly the effect of the approximation error $\varepsilon$ decays. Figure~\ref{fig:dependence_on_alpha} shows the coefficient for various $\alpha$. As $\alpha$ becomes higher, the decay slows.

\begin{figure}
  \includegraphics[width=\linewidth]{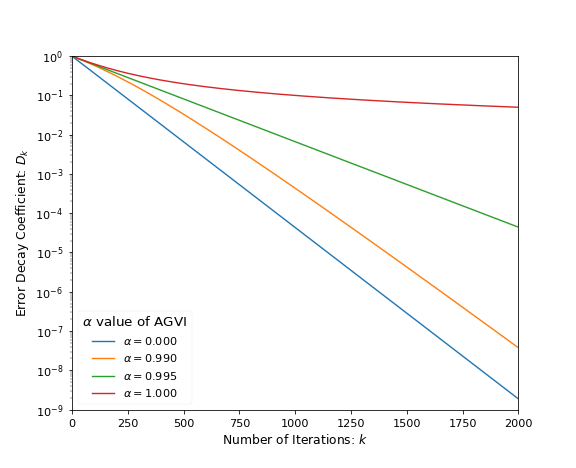}
  \caption{Comparison of approximation error decay in various $\alpha$ ($\gamma = 0.99$). Values of $\alpha$ are indicated by line colors as shown in the legend. Approximation error decays relatively quickly even if $\alpha=0.99$. On the other hand, when $\alpha=1$ (ADPP), the decay is very slow.} \label{fig:dependence_on_alpha}
\end{figure}

Accordingly, for some types of approximation error, such as model bias of a function approximator, $\varepsilon$ might pile up, and ADPP might perform poorly. Another source of such error is sampling bias due to a poor policy. In the beginning of learning, a policy that seems best is deployed to collect samples. However, such a policy may not be optimal, and may explore only a limited state and action space. As a result, approximation error is expected to accumulate outside the limited space. Over-estimation of Q-value function, which we explain next, is also a source of such error.

\subsubsection{Less Maximization Bias with finite $\beta$}
AVI tends to over-estimate the Q-value due to maximization bias. Such over-estimation can be caused not only by environmental stochasticity, but also by function approximation error. This is a significant problem when these algorithms are applied to complex RL tasks \cite{hasseltDoubleQ,hasseltDoubleDQN}.

To understand maximization bias, suppose that AVI has started with $Q_0$. When an environment or a policy is stochastic, $Q_1 (s, a)$ is a random variable. As a result, taking the maximum of $Q_1 (s, a)$ over actions corresponds to an estimator $\E \left[ \max_a Q_1(s, a) \right]$, i.e., over-estimation of $\max_a \E \left[ Q_1(s, a) \right]$, which we want in reality.

On the other hand, as $\beta \rightarrow 0$, the over-estimation diminishes. Indeed, since mellowmax is increasing in $\beta$, and convex in Q-value, we have
\begin{align*}
    \E \left[ \op{m}_0 Q_1(s) \right] \leq \op{m}_\beta \E \left[ Q_1(s, a) \right] \leq \E \left[ \op{m}_\beta Q_1(s) \right].
\end{align*}
The l.h.s. is equal to $\op{m}_0 \E \left[ Q_1(s, a) \right]$. Accordingly, for a small $\beta$, over-estimation of $\op{m}_\beta \E \left[ Q_1(s, a) \right]$ becomes less. Soft-update similar to the above works better than double Q-learning \cite{foxGLearning}.

\subsection{Derivation of the Algorithm}\label{subsec:roles_of_params}
To understand the meaning of $\alpha$ and $\beta$, we explain how GVI is derived from a regularized policy search method. The derivation is similar to that of DPP \cite{dpp}. A difference is that we use entropy regularization in addition to Kullback\textendash Leibler (KL) divergence.

\subsubsection{Regularized Policy Search to a New PI-Like DP}
Let $D(s; \pi, \widetilde{\pi}) := \sum_a \pi(a|s) \left( \log \pi(a|s) - \log \widetilde{\pi}(a|s) \right)$ denote KL divergence between policies $\pi$ and $\widetilde{\pi}$ at state $s$, and $H(s; \pi) := \sum_a \pi(a|s) \log \pi(a|s)$ denote entropy of $\pi$ at state $s$. Suppose a modified state value function
\begin{align}\label{eq:modified_state_value_function}
    &V_{\widetilde{\pi}}^{\pi} (s)= V^\pi(s) \nonumber\\
    &- \E^\pi \left[ \sum_{t \geq 0} \gamma^t \left( \frac{1}{\eta} D(s_t; \pi, \widetilde{\pi}) - \frac{1}{\theta} H(s_t; \pi) \right) \mid s_0 = s \right].
\end{align}

Let $\pi^\circ$ denote an optimal policy that maximizes the modified state value function above. It turns out that they have the following form. (A proof is in Appendix~\ref{sec:proof_derivation_of_gvi})
\begin{theorem}[Expression of an Optimal Policy $\pi^\circ$]\label{theorem:expression_of_optimal_policy}
For a modified state value function (\ref{eq:modified_state_value_function}), there exists a policy $\pi^\circ$ s.t. for any policy $\pi$, $V_{\widetilde{\pi}}^{\pi^\circ} \geq V_{\widetilde{\pi}}^{\pi}$. Furthermore, $V_{\widetilde{\pi}}^{\pi^\circ}$ and $\pi^\circ$ have the following form:
\begin{align*}
    V_{\widetilde{\pi}}^{\pi^\circ} (s)
    &= \frac{1}{\beta} \log \sum_a \widetilde{\pi}(a|s)^\alpha \exp \left( \beta Q_{\widetilde{\pi}}^{\pi^\circ} ( s, a) \right)\\
    \pi^\circ (a|s)
    &= \frac{\widetilde{\pi}(a|s)^\alpha \exp \left( \beta Q_{\widetilde{\pi}}^{\pi^\circ} (s, a) \right)}{\sum_{a'} \widetilde{\pi}(a'|s)^\alpha \exp \left( \beta Q_{\widetilde{\pi}}^{\pi^\circ} (s, a') \right)}\\
    &= \frac{\widetilde{\pi}(a|s)^\alpha \exp \left( \beta Q_{\widetilde{\pi}}^{\pi^\circ} (s, a) \right)}{\exp \left( \beta V_{\widetilde{\pi}}^{\pi^\circ} (s) \right)}.
\end{align*}
where $\alpha = \frac{\theta}{\theta +  \eta}$, $\beta = \frac{\theta \eta}{\theta + \eta}$, and $Q_{\widetilde{\pi}}^{\pi^\circ} := r + \gamma P V_{\widetilde{\pi}}^{\pi^\circ}$.
\end{theorem}
Therefore, after obtaining $V_{\widetilde{\pi}}^{\pi^\circ}$, $\pi^\circ$ can be computed with $V_{\widetilde{\pi}}^{\pi^\circ}$. Since $\pi^\circ$ maximizes expected cumulative rewards while maintaining entropy and KL divergence between $\pi^\circ$ and $\widetilde{\pi}$, $\pi^\circ$ is expected to be better than $\widetilde{\pi}$, but not to be too different from it and not to be deterministic.

We are interested in solving an original MDP. A straightforward approach is updating $\widetilde{\pi}$ to $\pi^\circ$, and finding a new optimal policy with $\pi^\circ$ as a new baseline policy, s.t.~KL divergence becomes $0$. This can be done by first obtaining $V_{\widetilde{\pi}}^{\pi^\circ}$ using fixed-point iteration
\begin{align*}
    V_{\widetilde{\pi}}^{k+1} (s)
    = \frac{\log \sum_a \widetilde{\pi}(a|s)^{\alpha} \exp \left( \beta Q_{\widetilde{\pi}}^k (s, a) \right)}{\beta},
\end{align*}
where $Q_{\widetilde{\pi}}^k (s, a) := r (s, a) + \gamma P V_{\widetilde{\pi}}^k (s, a)$. Then, we compute $\pi^\circ$ with $\limk V_{\widetilde{\pi}}^{k}$, and finally update $\widetilde{\pi}$ to $\pi^\circ$. By repeating these steps, the policy is expected to converge to an entropy-regularized optimal policy.

\subsubsection{Regularized Policy Search to a New VI-Like DP}
Rather than updating $V_{\widetilde{\pi}}^{k}$ infinitely, updating once might be enough, as is the case for VI. Suppose $V_0 \in \set{B}_{\set{S} \times \set{A}}$. In this case, update rule is
\begin{align}\label{eq:v_update}
    V_{k+1} (s)
    = \frac{\log \sum_a \pi_k(a|s)^{\alpha} \exp \left( \beta \left( r + \gamma P V_k \right) (s, a) \right)}{\beta},
\end{align}
where $\pi_0$ is an arbitrary policy satisfying $\pi(a|s) > 0$ for any state $s$ and action $a$, and
\begin{align*}
    \pi_{k+1} (a|s)
    = \pi_k(a|s)^\alpha \frac{\exp \left( \beta \left( r + \gamma P V_{k+1} \right)(s, a) \right)}{\exp \left( \beta V_{k+1} (s)\right)}.
\end{align*}

It turns out that (slightly modified version of) the above algorithm can be efficiently implemented by GVI. The modification is that policy improvement is done by
\begin{align}\label{eq:modified_policy_improvement_main}
    \pi_{k+1} (a|s)
    = \pi_k(a|s)^\alpha \frac{\exp \left( \beta \left( r + \gamma P V_{k} \right)(s, a) \right)}{\exp \left( \beta V_{k+1} (s)\right)}.
\end{align}
With this modification, GVI can be derived as follows: define $Q_k$ by\footnote{$\log | \set{A} |$ is added just for obtaining log-average-exp expression in the end. Without it, the almost same algorithm can be derived.}
\begin{align}\label{eq:q_k_action_preference_main}
    Q_{k+1}
    &:= r + \gamma P V_k + \frac{\alpha}{\beta} \log \pi_k + \frac{\alpha - \gamma}{(1-\gamma) \beta}\log |\set{A}|.
\end{align}
Equivalently, we have
\begin{align}\label{eq:r_pvk_main}
    r + \gamma P V_k = Q_{k+1} - \frac{\alpha}{\beta} \log \pi_k - \frac{\alpha - \gamma}{(1-\gamma) \beta}\log |\set{A}|.
\end{align}
By using (\ref{eq:v_update}) and (\ref{eq:r_pvk_main}),
\begin{align}
    &V_{k+1}(s)\nonumber\\
    &= \frac{\log \sum_a \exp \left( \beta Q_{k+1} (s, a) - \frac{\alpha - \gamma}{1 - \gamma} \log |\set{A}| \right)}{\beta}\\
    &= \frac{\log \sum_a \exp \left( \beta Q_{k+1} (s, a) \right)}{\beta} - \frac{\alpha - \gamma}{ (1 - \gamma) \beta} \log |\set{A}|\\
    &= \op{m}_\beta Q_{k+1} (s) + \frac{1 - \alpha}{ (1 - \gamma) \beta } \log |\set{A}|.\label{eq:vk_expression_main}
\end{align}
Therefore, we have
\begin{align*}
    &r + \gamma P V_{k} - V_{k+1}\nonumber\\
    &=Q_{k+1} - \frac{\alpha}{\beta} \log \pi_k - \frac{\alpha - \gamma}{(1-\gamma) \beta}\log |\set{A}| - V_{k+1}\nonumber\\
    &=Q_{k+1} - \frac{\alpha}{\beta} \log \pi_k - \frac{\log \sum_a \exp \left( \beta Q_{k+1} (\cdot, a) \right)}{\beta}.
\end{align*}
Consequently, by substituting $r + \gamma P V_{k} - V_{k+1}$ in (\ref{eq:modified_policy_improvement_main}) with the above expression,
\begin{align}
    \pi_{k+1}(a|s)
    &= \frac{\exp \left( \beta Q_{k+1} (s, a) \right)}{\sum_{a'} \exp \left( \beta Q_{k+1} (s, a') \right)}.\label{eq:policy_expression_main}
\end{align}
Plugging back (\ref{eq:vk_expression_main}) and (\ref{eq:policy_expression_main}) to $V_k$ and $\pi_k$ in (\ref{eq:q_k_action_preference_main}), respectively, we get
\begin{align*}
    Q_{k+1} (s, a)
    &= r(s, a) + \gamma P \op{m}_\beta Q_k (s, a) + \alpha Q_k(s, a) \nonumber\\
    &\hspace{2mm}- \frac{\alpha}{\beta} \log \sum_{a'} \exp \left( \beta Q_k (s, a')\right) + \frac{\alpha}{\beta} \log |\set{A}|\\
    &=\op{T}_\beta Q_k (s, a) + \alpha \left[ Q_k(s, a) - \op{m}_\beta Q_k (s, a) \right].
\end{align*}
The last line exactly corresponds to GVI update rule. 

\section{Numerical Experiments}
Our purposes in the numerical experiments are the followings:
\begin{itemize}
    \item \textbf{Purpose 1.} We confirm that Theorem~\ref{theorem:generalizedVI} is consistent with numerical experiments, and that the Q-value difference can be enhanced as $\alpha$ approaches $1$.
    \item \textbf{Purpose 2.} $\alpha = 1$ (or ADPP) may need time to switch from a poor initial policy to a better policy. We examine whether by setting $\alpha$ to a moderate value, such a problem can be ameliorated.
    \item \textbf{Purpose 3.} $\beta = \infty$ (or AAL) over-estimates Q-values. We examine whether by setting $\beta$ to a moderate value, such a problem can be avoided.
\end{itemize}

\begin{figure}[t]
  \includegraphics[width=\linewidth]{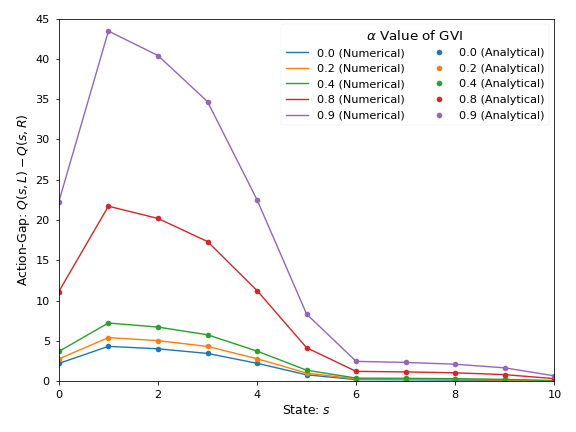}
  \caption{Numerical and analytical values of action-gap $Q(s, L) - Q(s, R)$ at each state after $100,000$ iterations with various $\alpha$. Values of $\alpha$ are indicated by colors as shown in the legend. Lines indicate numerical values, and dots indicate analytical values predicted by Theorem~\ref{theorem:generalizedVI}. The prediction is accurate, and the action-gap is enhanced.} \label{fig:action_gap_enhancement}
\end{figure}

\subsection{Environments and Experimental Conditions}
We used the following environments.

\subsubsection{ChainWalk} There are $11$ states ($0, 1, ..., 10$) connected like a chain, and the agent can move either left or right. Training episodes always start from state $5$. State transition to a desired direction occurs with probability $0.7$. With probability $0.3$, state transition to the opposite occurs. At the ends of the chain, attempted movement to outside of the chain results in staying at the ends. When an agent gets to a state which is on the left (or right) side, but not at the left (or right) end of the chain, the agent gets $-1$ (or $1$) reward. If the agent reaches the center, or state $5$, it gets no reward. If the agent moves to the left (or right) end of the chain, it can get $3$ (or $1$) reward. In this environment, optimal behavior is going to the left regardless of states. For brevity, we denote the Q-value of going left by $Q(s, L)$, and right by $Q(s, R)$ in this environment.

\subsubsection{LongChainWalk} The LongChainWalk environment is a modified version of the ChainWalk environment. We modified the environment as follows: First, the chain consisted of $51$ states. Second, training episodes start from a uniformly sampled state. Third, actions are specified by integers from $-5$ to $5$ meaning a desired movement to state $s + a$, where $s$ is a current state, and $a$ is an action. In other words, the agent is able to make larger movements. Since the over-estimation problem becomes more serious as the number of actions increases, this modification is important for our purpose. Fourth, action always succeeds, but a subsequent state is $s' = clip(s + a + n)$, where $n$ is sampled uniformly from an integer from $-3$ to $3$ at every state transition, and $clip(x)$ restrict $x$ to $[0, 50]$. Finally, immediate reward is $\exp (-(s'-25)^2 / 25)$, where $s'$ is a subsequent state. Therefore, the agent needs to move toward the center.

In an experiment for Purpose~1, the ChainWalk environment was used. We updated a Q-table with a perfect model of the environment. $\beta$ and $\gamma$ are fixed to $10$ and $0.99$, respectively. 

In an experiment for Purpose~2, we again used the ChainWalk environment. However, this time, we trained an agent without the environmental model. Training consisted of $2,500$ episodes. In each episode, the agent was allowed to take $100$ actions according to $\epsilon$-greedy. After every episode, the Q-table was immediately updated using experience the agent obtained during the episode. After every episode, evaluation of the agent was performed. The evaluation consisted of $100$ episodes starting from a state sampled uniformly. The agent was allowed to take greedy actions w.r.t. the Q-value it obtained from the training. The metric of the agent is the median of mean episodic rewards in an evaluation over $20$ experimental runs. $\beta$ and $\gamma$ are fixed to $10$ and $0.99$, respectively.

In an experiment for Purpose~3, we used LongChainWalk. Except that training consisted of $5,000$ episodes, and that $\varepsilon$ was fixed to $0.3$, training conditions were same as the second experiment.

\subsection{Value Difference Enhancement (Purpose~1)}
Figure~\ref{fig:action_gap_enhancement} compares the numerical and analytical values of action-gap $Q(s, L) - Q(s, R)$ at various $\alpha$. It shows that Theorem~\ref{theorem:generalizedVI} is consistent with numerical experiments, and the action-gap increases as $\alpha$ approaches $1$, as predicted.

\begin{figure}[t]
  \includegraphics[width=\linewidth]{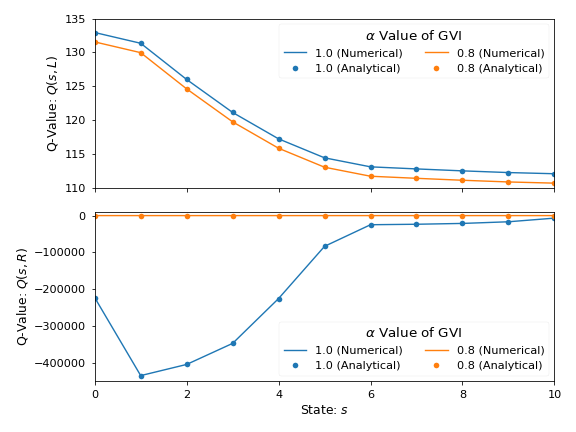}
  \caption{Numerical and analytical Q-values after $100,000$ iterations when $\alpha = 1.0$ (and $0.8$ just for comparison). The upper row shows the Q-value for going left, and the lower row shows the Q-value for going right. Lines indicate actual Q-values, and dots indicate analytical values predicted by Theorem~\ref{theorem:generalizedVI}. The prediction is accurate, and the action-gap is strongly enhanced (and eventually diverges).} \label{fig:q_value_when_alpha_1}
\end{figure}

Figure~\ref{fig:q_value_when_alpha_1} shows the numerical and analytical Q-values after $100,000$ iterations when $\alpha = 1.0$. In this environment, going right is a sub-optimal action. Therefore, it is strongly devalued when $\alpha = 1.0$ ($Q(s,R)$ is diverging to $-\infty$).

From these results, we conclude that Theorem~\ref{theorem:generalizedVI} is consistent with numerical experiments, and that the Q-value difference can be increasingly enhanced as $\alpha$ approaches to $1$.

\subsection{Error Decay Property of AGVI (Purpose~2)}
$\alpha = 1$ (or ADPP) may need time to switch from a poor initial policy (going right) to a better policy (going left). We examine whether by setting moderate $\alpha$, such a problem can be ameliorated.

\begin{figure}[t]
  \includegraphics[width=\linewidth]{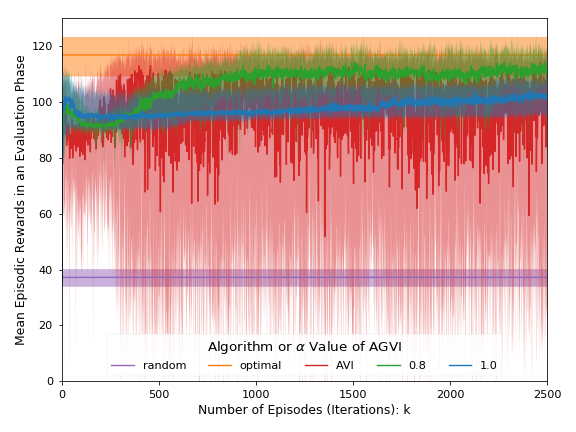}
  \caption{Performance comparison of AGVI with different $\alpha$. Values of $\alpha$ are indicated by colors, as shown in the legend. "AVI" means the result of approximate value iteration. "Optimal" means that of an optimal policy, and "random" means that of a random policy. Lines indicate the median (over $20$ experimental runs) of mean episodic rewards in an evaluation phase. The shaded area is the $90$ percentile. Among all results, $\alpha = 0.8$ works best. When $\alpha$ is either $0.8$ or $1.0$, the area of the $90$ percentile is narrower compared to AVI, showing the robustness of AGVI. When $\alpha = 1.0$, the performance approaches that of $\alpha=0.8$ to the end. However, learning is slow.} \label{fig:error_decay_comparison_performance}
\end{figure}

Figure~\ref{fig:error_decay_comparison_performance} shows the result. When $\alpha = 1.0$, the performance is poorer than that of $\alpha = 0.8$. However, performance slowly approaches that of $\alpha=0.8$ as learning proceeds. For reasonably large $\beta$ and $\epsilon$ smaller than or equal to approximately $0.8$, similar results were obtained.

\begin{figure}[t]
  \includegraphics[width=\linewidth]{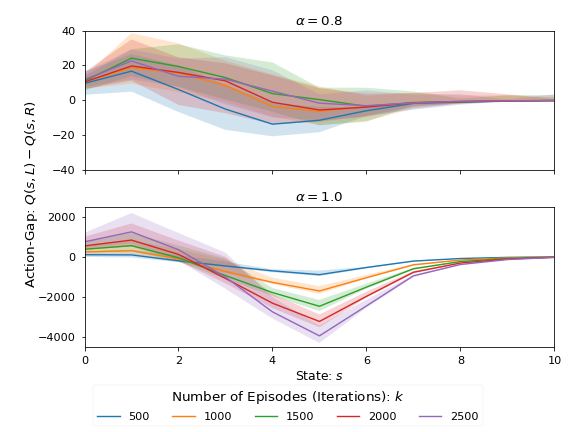}
  \caption{Action-gap $Q(s, L) - Q(s, R)$ at various iterations when $\alpha=1.0$ (lower, and  $0.8$ (upper) just for comparison). Lines indicate the median (over $20$ experimental runs) action-gap at different iterations. The shaded area is the $90$ percentile.  Numbers of iterations are indicated by colors as shown in the legend. In the beginning, the Q-value of going left was strongly devalued since it seemed to be sub-optimal. Information that going left is actually optimal slowly propagates toward the center, but it takes a very long time to overcome the initial wrong Q-value. As a result, even after $2500$ iterations, optimal actions are not found in a broad region due to prolonged incorrect devaluation of the optimal action.} \label{fig:error_decay_comparison_q_value}
\end{figure}

In order to further analyze what was occurring, we visualized the Q-value of ADPP (Fig.~\ref{fig:error_decay_comparison_q_value}). It suggested that slow learning when $\alpha = 1$ (ADPP) is caused by prolonged devaluation of an optimal action.

In summary, as Fig.~\ref{fig:error_decay_comparison_performance} shows, when $\alpha=1$ (ADPP), learning is slow. This occurred because ADPP takes a long time to switch from an initial poor policy that tends to go right to better policy that tends to go left (Fig.~\ref{fig:error_decay_comparison_q_value}) by a strong marginalization of a sub-optimal action. Indeed, this slow learning was not seen when $\epsilon$ was higher, supposedly due to almost exploratory behavior. This policy switching is probably important in complex environments in which an initial policy is likely to be sub-optimal. Figure~\ref{fig:error_decay_comparison_performance} shows that setting $\alpha$ to a moderate value ameliorates this problem while outperforming AVI.

\subsection{Less Maximization Bias (Purpose~3)}
Finally, we conducted an experiment to investigate whether by setting $\beta$ to a moderate value, over-estimation of the Q-value by maximization bias could be avoided.

\begin{figure}[t]
  \includegraphics[width=\linewidth]{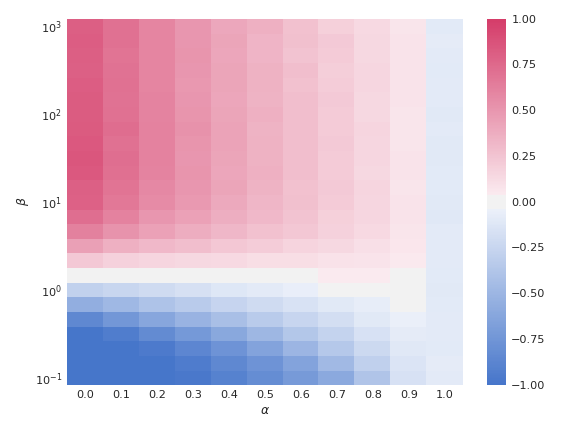}
  \caption{Error ratio (ER) defined by (\ref{eq:error_ratio}) across different parameter settings. Color indicates the mean (over $20$ experimental runs) of ER at the last iteration. Higher ER indicates over-estimation, and lower ER indicates under-estimation. The top left corner and the top edge correspond approximately to AVI and AAL, respectively. The right edge corresponds to ADPP. It is clear that a small $\beta$ leads to less over-estimation while higher $\beta$ leads to over-estimation. Interestingly, higher $\alpha$ leads to less over-estimation.} \label{fig:q_val_beta}
\end{figure}

We define the error ratio (ER) by
\begin{align}\label{eq:error_ratio}
ER := \frac{\sum_{s} \left[ \max_a \widetilde{Q}(s, a) - \max_a Q(s, a) \right]}{\left| \sum_{s} \max_a Q(s, a) \right|},
\end{align}
where $Q$ is the true Q-value of a corresponding parameter setting, and $\widetilde{Q}$ is the estimated Q-value. Therefore, stronger over-estimation results in a higher error ratio. In Fig.~\ref{fig:q_val_beta}, ERs for various parameters are shown. One can see that over-estimation appears when $\beta$ is high. In particular, parameter settings at the upper left corner strongly over-estimate the Q-value.

\begin{figure}[t]
  \includegraphics[width=\linewidth]{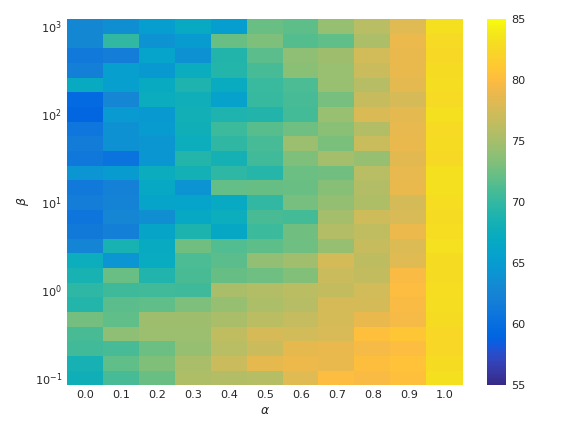}
  \caption{Performance comparison across different parameter settings. Color indicates the mean (over $20$ experimental runs) of episodic rewards in the last evaluation phase. The highest value $85$ corresponds approximately to the episodic rewards of an optimal policy. A random policy attains around $15$ episodic rewards. The top left corner and the top edge correspond approximately to AVI and AAL, respectively. The right edge corresponds to ADPP. It is clear that a moderate value of $\beta$ works best for any $\alpha$.} \label{fig:performace_comparison_beta}
\end{figure}

Figure~\ref{fig:performace_comparison_beta} shows final performance across different parameter settings. It is clear that a moderate value of $\beta$ works best, for except for $\alpha=1$, where no over-estimation occurred, and performance was high with any choice of $\beta$. However, even in this simple environment, we observed that intermediate performance was higher when $\alpha$ was set to a moderate value such as $0.8$. This observation is consistent with the finding that a moderate $\alpha$ may lead to faster learning because of faster error decay.

From these results, we conclude that by setting $\beta$ to a moderate value, over-estimation can be avoided. Rather, under-estimation occurs.

\section{Related Work}
Our work was stimulated by research that established connections between a regularized policy search and value-function learning such as \cite{dpp,pgq,foxGLearning,pcl}. In particular, our work is an extension of \cite{dpp}, further connecting a regularized policy search with AL \cite{increaseActionGap}.

In \cite{increaseActionGap}, it is shown that a class of operators may enhance the action-gap. In this paper, we showed the value to which AL converges. We also showed a performance guarantee that implies that as $\alpha$ increases, AAL becomes robust to approximation errors although its learning slows. A performance guarantee for AAL is new, and our performance guarantee explains why AAL works well.

Over-estimation of the Q-value by maximization bias was first noted by \cite{thrun_maximization_bias}, and several researchers have addressed it in various ways \cite{foxGLearning,hasseltDoubleQ,hasseltDoubleDQN}. In particular, GVI is similar to a batch version of G-learning \cite{foxGLearning} with action-gap enhancement. Fox et al. predicted that action-gap enhancement would further ameliorate maximization bias. Our experimental results support their argument. Another approach to tackle the over-estimation uses a double estimator, as proposed in \cite{hasseltDoubleQ}. With a double estimator, searching to find optimal $\beta$ or scheduling of $\beta$ can be avoided. This may be a good choice when interactions with an environment require a long time so that short experiments with different $\beta$ are prohibitive. However, the use of a double estimator doubles the sample complexity of algorithms. In \cite{foxGLearning}, the authors showed that soft-update using appropriate scheduling of $\beta$ led to faster and better learning.

Recently, a unified view of a regularized policy search and DPP has been provided \cite{unified_policy_search_and_DPP}. Our work is limited in that we only unified value-iteration-like algorithms. However, our work shows that AL can be also seen in the unified view. In addition, our work is more advanced in that we showed a performance guarantee for AGVI, which includes AAL, for which there has been no performance guarantee previously.

\section{Conclusion}
In this paper, we proposed a new DP algorithm called GVI, which unifies VI, AL, and DPP. We showed a performance guarantee of its approximate version, AGVI, and discussed a weakness of ADPP. We also showed that AAL tends to over-estimate Q-values. Experimental results support our argument, and suggest our algorithm as a promising alternative to existing algorithms. Specifically, AGVI allows us to balance (i) faster learning and robustness to approximation error and (ii) maximization bias and optimality of the algorithm. We also showed an interesting connection between GVI and a regularized policy search. For AL, such a connection was formerly unknown.

\section*{Acknowledgement}
This work was supported by JSPS KAKENHI Grant Numbers 16H06563 and 17H06042.

\bibliographystyle{aaai}
\bibliography{main}

\newpage

\appendix
\section{Lemmas Related to Mellowmax}\label{sec:useful_lemmas}
We prove lemmas related to mellowmax. They are used throughout the proof of Theorem~\ref{theorem:generalizedVI} and Theorem~\ref{theorem:bound_agvi}. For brevity, we use the following definitions in this section:
\begin{align*}
    f(\beta) &:= \frac{\sum_i x_i \exp \left( \beta x_i \right) }{\sum_j \exp \left( \beta x_j \right)}\\
    g(\beta) &:= \frac{1}{\beta} \log \frac{\sum_i \exp \left( \beta x_i \right)}{N},
\end{align*}
where $x_i \in \R$, and $i=1, \ldots, N$.

\begin{lemma}\label{lemma:mellowmax_smaller_than_boltzmann}
$\forall \beta > 0$, $\beta^{-1} \log N \geq f(\beta) - g(\beta) \geq 0$.
\end{lemma}
\begin{proof}
Indeed, consider an entropy $H(\beta)$ of
\begin{align*}
   p(i; \beta) := \frac{\exp \left( \beta x_i \right) }{\sum_j \exp \left( \beta x_j \right) }.
\end{align*}
It can be rewritten as
\begin{align*}
   H(\beta)
   &= \log \sum_i \exp \left( \beta x_i \right) - \beta f(\beta)\\
   &= \beta g(\beta) + \log N - \beta f(\beta).
\end{align*}
Accordingly, $f(\beta) - g(\beta) = \beta^{-1} \left[ \log N - H(\beta) \right]$. Since $0 \leq H(\beta) \leq \log N$, we can conclude the proof.
\end{proof}

\begin{lemma}\label{lemma:mellowmax_nondecreasing}
$\forall \beta > 0$, $g(\beta) + \beta^{-1} \log N$ is non-increasing, but $g(\beta)$ is non-decreasing.
\end{lemma}
\begin{proof}
Indeed,
\begin{align*}
   \frac{d g(\beta)}{d \beta}
    &= \frac{\sum_i x_i \exp ( \beta x_i ) }{\beta \sum_i \exp (\beta x_i) } - \frac{\log N^{-1} \sum_i \exp ( \beta x_i ) }{\beta^2}\\
    &= \frac{1}{\beta} \left[ f(\beta) - g(\beta) \right].
\end{align*}
Therefore, the derivative of $g(\beta) + \beta^{-1} \log N$ is smaller than or equal to $0$, but the derivative of $g(\beta)$ is larger than or equal to $0$.
\end{proof}

\section{Proof of Theorem~\ref{theorem:generalizedVI}.}\label{sec:proof_generalizedVI}
For shorthand notation, we use $\theta := \frac{\beta}{1 - \alpha}$. This value frequently appears as the unique fixed point of $\op{T}_\theta$, i.e., $Q^\theta$. When either $\alpha=1$ or $\beta=\infty$, $\theta$ needs to be understood as $\infty$, and $Q^\theta$ needs to be read as $Q^*$. Hereafter, we use this notation.

We mainly assume $\alpha \neq 1$. For $\alpha = 1$, take the limit of $\alpha \rightarrow 1$ appropriately (also see \cite{dpp}). For brevity, we define
\begin{equation*}
    A_k := \frac{1-\alpha^k}{1-\alpha}
\end{equation*}
for non-negative integer $k$. Note that when $\alpha \rightarrow 1$, $A_k = k$.

For later use in a proof of Theorem~\ref{theorem:bound_agvi}, we deal with a case where AGVI update is used in this section. Its update rule is the following: suppose $Q_0 \in \set{B}_{\set{S} \times \set{A}}, \| Q_0 \| \leq V_{max}$, $Q_k$ is obtained by applying to $Q_0$ the update rule of AGVI
\begin{equation*}
    Q_{k+1} = \op{T}_\beta Q_{k} + \alpha \left( Q_{k} - \op{m}_\beta Q_k \right) + \varepsilon_k,
\end{equation*}
where $\alpha \in [0, 1]$, $\beta \in (0, \infty]$, and $\varepsilon_k \in \set{B}_{\set{S} \times \set{A}}$ is approximation error at iteration $k$.

The following series of functions $q_k$ turns out to be very useful: $q_0 := Q_0$ and
\begin{equation*}
    A_{k+1} q_{k+1} := A_{k} \op{T}^{\pi_k} q_k + \alpha^k \left( \op{T}^{\pi_k} q_0 + E_k \right),
\end{equation*}
where $\pi_k$ satisfies $\pi_k Q_k = \op{m}_\beta Q_k$, and $E_k = \sum_{i=0}^{k} \alpha^{-i} \varepsilon_i$.

\subsection{Proof Sketch}\label{subsec:proof_sketch}
Since the proof is lengthy, we provide a sketch of the proof.

Lemma~\ref{lemma:q_k_bound} shows that $\| Q^\theta - q_{k+1} \|$ is bounded by
\begin{align}\label{eq:sketch_q_k_bound}
    &A_{k+1} \left\|Q^\theta - q_{k+1} \right\|\nonumber\\
    &\hspace{10mm}\leq \gamma A_k \left\| Q^\theta - q_k \right\| + \alpha^k \left( C_0 + \| E_k\| \right),
\end{align}
where $C_0 := 2 \gamma V_{max} + \gamma \alpha \beta^{-1} \log |\set{A}|$.

By using (\ref{eq:sketch_q_k_bound}), we can show that when there is no approximation error, $\limk q_k = Q^\theta$. To show this, we suppose that it does not hold, and deduce a contradiction. When there is no approximation error, (\ref{eq:sketch_q_k_bound}) becomes
\begin{align*}
    &\left\|Q^\theta - q_{k+1} \right\|\\
    &\hspace{5mm}\leq \left( \gamma \frac{A_{k}}{A_{k+1}} + \frac{\alpha^k C_0}{A_{k+1} \left\| Q^\theta - q_k \right\|} \right) \left\| Q^\theta - q_k \right\|.
\end{align*}
Since $A_k$ converges to $(1-\alpha)^{-1}$, $\gamma A_{k+1}^{-1} A_k$ converges to $\gamma$. On the other hand, $\alpha^k A_{k+1}^{-1}$ converges to $0$. Therefore,
\begin{equation*}
    \frac{\alpha^k C_0}{A_{k+1} \left\| Q^\theta - q_k \right\|} \rightarrow 0
\end{equation*}
unless $\left\| Q^\theta - q_k \right\| = O(\alpha^k)$ or $o(\alpha^k)$, both of which implies $\limk q_k = Q^\theta$ with a convergence rate equal to or faster than $\alpha^k$. Accordingly, it converges to $0$, and there exists $K$ such that for $k > K$,
\begin{align*}
    \left\|Q^\theta - q_{k+1} \right\|
    \leq c \left\| Q^\theta - q_k \right\|,
\end{align*}
where $c < 1$. It clearly follows that $q_k$ converges to $Q^\theta$, and it contradicts to the assumption that $\limk q_k = Q^\theta$ does not hold. Therefore, $\limk q_k = Q^\theta$.

Furthermore, from this discussion, one can see that there exists $K$ such that
\begin{align*}
    \left\| (k+t) Q^\theta - (k+t) q_{k+t} \right\|
    \leq c^t (k+t) \left\| Q^\theta - q_{k} \right\|.
\end{align*}
This shows that for any $\varepsilon \in \R^+$, there exists $K$ such that $k > K$, $\| k Q^{\theta} - k q_k \| < \varepsilon$.

Lemma~\ref{lemma:Q_k_expression} shows that $Q_k$ can be expressed by
\begin{equation}\label{eq:sketch_Q_k_expression}
    Q_{k} = A_k q_k + \alpha^k q_0 - \alpha \op{m}_\beta \left( A_{k-1} q_{k-1} + \alpha^{k-1} q_0 \right).
\end{equation}
Note that for any functions $f, g  \in \set{B}_{\set{S} \times \set{A}}$,
\begin{equation*}
    \left\| \op{m}_\beta f - \op{m}_\beta g \right\| \leq \left\| f - g \right\|
\end{equation*}
holds \cite{asadi_mellowmax}. Therefore, for any sequence of functions $f_n \in \set{B}_{\set{S} \times \set{A}}$ such that $\limk f_k = f$, $\limk \op{m}_\beta f_k = \op{m}_\beta f$. Accordingly,
\begin{align*}
    \limk \op{m}_\beta \left( A_{k-1} q_{k-1} + \alpha^{k-1} q_0 \right)
    &= \op{m}_\beta \left( \frac{Q^\theta}{1-\alpha} \right)\\
    &= \frac{\op{m}_\theta Q^\theta}{1-\alpha},
\end{align*}
and $\limk Q_k = \op{m}_{\theta} Q^\theta + \frac{1}{1-\alpha} \left( Q^\theta - \op{m}_{\theta} Q^\theta \right)$.

\subsection{Proofs}
\begin{lemma}\label{lemma:Q_k_expression}
\begin{equation}\label{eq:Q_k_expression}
    Q_{k} = A_k q_k + \alpha^k q_0 - \alpha \pi_{k-1} \left( A_{k-1} q_{k-1} + \alpha^{k-1} q_0 \right).
\end{equation}
\end{lemma}

\begin{proof}
We prove the claim by induction. For $k=0$,
\begin{align*}
    Q_1
    &= \op{T}_\beta Q_0 + \alpha \left( Q_0 - \op{m}_\beta Q_0 \right) + \varepsilon_0\\
    &= \op{T}^{\pi_0} q_0 + \varepsilon_0 + \alpha \left( q_0 - \pi_0 q_0 \right)\\
    &= q_1 + \alpha \left( q_0 - \pi_0 q_0 \right)\\
    &= A_1 q_1 + \alpha^1 q_0 - \alpha \pi_0 \left( A_0 q_0 + \alpha^{0} q_0 \right).
\end{align*}

Next, suppose that up to $k$, the claim holds. Then,
\begin{align*}
    &\op{T}_\beta Q_k\\
    &= \op{T}^{\pi_k} \left[ A_k q_k + \alpha^k q_0 - \alpha \pi_{k-1} \left( A_{k-1} q_{k-1} + \alpha^{k-1} q_0 \right) \right]\\
    &= \left( A_k + \alpha^k - \alpha A_{k-1} - \alpha^{k} \right) r\\
    &+ \gamma P^{\pi_k} \left[ A_k q_k + \alpha^k q_0 - \alpha \pi_{k-1} \left( A_{k-1} q_{k-1} + \alpha^{k-1} q_0 \right) \right]\\
    &= A_k \op{T}^{\pi_k} q_k + \alpha^k \op{T}^{\pi_k} q_0\\
    &\hspace{5mm}- \alpha A_{k-1}\op{T}^{\pi_{k-1}} q_{k-1} - \alpha^k \op{T}^{\pi_{k-1}} q_0\\
    &= A_{k+1} q_{k+1} - \alpha A_k q_{k} - \varepsilon_k.
\end{align*}
Furthermore,
\begin{align*}
    Q_k - \pi_k Q_k = A_k q_{k} + \alpha^k q_0 - \pi \left( A_k q_{k} + \alpha^k q_0 \right).
\end{align*}
Therefore
\begin{align*}
    &Q_{k+1}\nonumber\\
    &= \op{T}_\beta Q_k + \alpha \left( Q_k - \op{m}_\beta Q_k \right) + \varepsilon_k\\
    &= A_{k+1} q_{k+1} + \alpha^{k+1} q_0 - \alpha \pi \left( A_k q_{k} + \alpha^k q_0 \right),
\end{align*}
This concludes the proof.
\end{proof}

The following two claims show to what value $q_k$ is converging when there is no approximation error.
\begin{lemma}\label{lemma:q_k_bound}
$\| Q^\theta - q_{k+1} \|$ is bounded by
\begin{align*}
    A_{k+1} \left\|Q^\theta - q_{k+1} \right\|
    \leq \gamma A_k \left\| Q^\theta - q_k \right\| + \alpha^k \left( C_0 + \| E_k\| \right),
\end{align*}
where $C_0 := 2 \gamma V_{max} + \gamma \alpha \beta^{-1} \log |\set{A}|$.
\end{lemma}

\begin{proof}
By definition,
\begin{align*}
    q_{k+1}
    &= \frac{A_k}{A_{k+1}} \op{T}^{\pi_k} q_k + \frac{\alpha^k}{A_{k+1}} \left( \op{T}^{\pi_k} q_0 + E_k \right)\\
    &= r + \frac{\gamma}{A_{k+1}} P^{\pi_k} \left( A_k q_k + \alpha^k q_0 \right) + \frac{\alpha^k E_k}{A_{k+1}},
\end{align*}
where we used a fact that $A_{k+1}^{-1} A_k + A_{k+1}^{-1} \alpha^k = 1$ to exchange the order of $\op{T}^{\pi_k}$ and the coefficients. From (\ref{eq:Q_k_expression}),
\begin{align*}
    \pi_k Q_k
    &= \pi_k \left( A_k q_k + \alpha^k q_0 \right)\\
    &\hspace{10mm} - \alpha \pi_{k-1} \left( A_{k-1} q_{k-1} + \alpha^{k-1} q_0 \right)\\
    &= \op{m}_\beta \left( A_k q_k + \alpha^k q_0 \right)\\
    &\hspace{10mm}  - \alpha \pi_{k-1} \left( A_{k-1} q_{k-1} + \alpha^{k-1} q_0 \right)\\
    &= \op{m}_\beta Q_k.
\end{align*}
Therefore,
\begin{align}\label{eq:pi_k_m_k}
    \pi_k \left( A_k q_k + \alpha^k q_0 \right) = \op{m}_\beta \left( A_k q_k + \alpha^k q_0 \right).
\end{align}
Hence, we have
\begin{align*}
    q_{k+1}
    &= r + \frac{\gamma}{A_{k+1}} P \op{m}_\beta \left( A_k q_k + \alpha^k q_0 \right) + \frac{\alpha^k E_k}{A_{k+1}}\\
    &= \op{T}_{A_{k+1} \beta} \left( \frac{A_{k}}{A_{k+1}} q_k + \frac{\alpha^k}{A_{k+1}} q_0 \right) + \frac{\alpha^k E_k}{A_{k+1}}.
\end{align*}

By defining $\delta_k := \frac{\alpha^{k+1}}{1-\alpha} \beta$ (in case of $\alpha=1$, read $\theta$ as $\infty$ and $\theta - \delta_k$ as $k \beta$),
\begin{align}\label{eq:q_k}
    q_{k+1}
    = \op{T}_{\theta-\delta_k} \left( \frac{A_{k}}{A_{k+1}} q_k + \frac{\alpha^k}{A_{k+1}} q_0 \right) + \frac{\alpha^k E_k}{A_{k+1}}.
\end{align}
Accordingly,
\begin{align}
    &\left\|Q^\theta - q_{k+1} \right\| \nonumber\\
    &\leq \left\|\op{T}_\theta Q^\theta - \op{T}_\theta \left( \frac{A_{k}}{A_{k+1}} q_k + \frac{\alpha^k}{A_{k+1}} q_0 \right) \right\|\nonumber\\
    &+ \left\| \op{T}_\theta \left( \frac{A_{k}}{A_{k+1}} q_k + \frac{\alpha^k}{A_{k+1}} q_0 \right) - q_{k+1} \right\|\nonumber\\
    &\leq \gamma \frac{A_{k}}{A_{k+1}} \left\| Q^\theta - q_k \right\| + \gamma \frac{\alpha^k}{A_{k+1}} \left\| Q^\theta - q_0 \right\| \nonumber\\
    &+ \left\| \op{T}_\theta \left( \frac{A_{k}}{A_{k+1}} q_k + \frac{\alpha^k}{A_{k+1}} q_0 \right) - q_{k+1} \right\|,\label{eq:temp_q_m_bound}
\end{align}
where again we used a fact that $A_{k+1}^{-1} A_k + A_{k+1}^{-1} \alpha^k = 1$ at the last inequality.

As shown in Lemma~\ref{lemma:mellowmax_nondecreasing}, $\op{m}_\beta h (s) + \beta^{-1} \log |\set{A}|$, where $h \in \set{B}_{\set{S} \times \set{A}}$, is non-increasing in $\beta$. Hence, 
\begin{align*}
    0 \geq \op{m}_{\theta} h(s) + \frac{\log |\set{A}|}{\theta} - \op{m}_{\theta - \delta_k} h(s) - \frac{\log |\set{A}|}{\theta - \delta_k}.
\end{align*}
Therefore,
\begin{align*}
    0
    \leq \op{m}_{\theta} h(s) - \op{m}_{\theta - \delta_k} h(s)
    &\leq \frac{\log |\set{A}|}{\theta - \delta_k} - \frac{\log |\set{A}|}{\theta}\\
    &= \frac{\alpha^{k+1}}{A_{k+1}} \frac{\log |\set{A}|}{\beta}.
\end{align*}
Hence, by substituting $q_{k+1}$ with (\ref{eq:q_k}),
\begin{align*}
    &\left\| \op{T}_\theta \left( \frac{A_{k}}{A_{k+1}} q_k + \frac{\alpha^k}{A_{k+1}} q_0 \right) - q_{k+1} \right\|\nonumber\\
    &\hspace{10mm}\leq \gamma \left\| \op{m}_{\theta} q' - \op{m}_{\theta-\delta_k} q' \right\| + \frac{\alpha^k \left\| E_k \right\|}{A_{k+1}}\\
    &\hspace{10mm}\leq \gamma \frac{\alpha^{k+1}}{A_{k+1}} \frac{\log |\set{A}|}{\beta} + \frac{\alpha^k \left\| E_k \right\|}{A_{k+1}},
\end{align*}
where $q' = A_{k+1} A_k^{-1} q_k + A_{k+1} \alpha^k q_0$.

By multiplying both sides of (\ref{eq:temp_q_m_bound}) by $A_{k+1}$, we obtain
\begin{align*}
    A_{k+1} \left\|Q^\theta - q_{k+1} \right\|
    \leq \gamma A_k \left\| Q^\theta - q_k \right\| + \alpha^k \left( C_0 + \| E_k\| \right),
\end{align*}
where $\left\| Q^\theta - q_0 \right\| \leq 2 V_{max}$ is used. This concludes the proof.
\end{proof}

As a corollary of Lemma~\ref{lemma:q_k_bound}, we have the following.
\begin{corollary}
When there is no approximation error, $q_k$ converges uniformly to $Q^\theta$.
\end{corollary}

\begin{proof}
The claim for $\alpha \neq 1$ has been shown in Sect.~\ref{subsec:proof_sketch}. For $\alpha = 1$, read $\theta$ as $\infty$ and $A_k$ as $k$. Then, the almost same proof can be used.
\end{proof}

Now, we prove Theorem~\ref{theorem:generalizedVI}.
\begin{proof}[Proof of Theorem~\ref{theorem:generalizedVI}]
We first prove the claim for $\alpha \neq 1$. From (\ref{eq:pi_k_m_k}),
\begin{equation*}
    Q_k
    = A_k q_k + \alpha^k q_0 - \alpha \op{m}_\beta \left( A_{k-1} q_{k-1} + \alpha^{k-1} q_0 \right),
\end{equation*}
and the explanation given in Sect.~\ref{subsec:proof_sketch}, the claim holds.

For $\alpha=1$, (\ref{eq:Q_k_expression}) is
\begin{equation*}
    Q_{k} = k q_k + q_0 - \alpha \op{m}_\beta \left[ (k-1) q_{k-1} + q_0 \right].
\end{equation*}
As explained in Sect.~\ref{subsec:proof_sketch}, the rate of convergence of $q_k$ is faster than linear convergence rate. Therefore, $\forall \varepsilon \in \R^+$, $\exists K$ such that $\forall k > K$, $\exists \psi \in \set{B}_{\set{S} \times \set{A}}, \| \psi \| < \varepsilon$ such that
\begin{equation*}
    k q_k = k Q^* + \psi.
\end{equation*}
Accordingly,
\begin{align*}
    Q_k = k Q^* + q_0 - \op{m}_{\beta} \left( (k - 1) Q^* + q_0 \right) + \phi
\end{align*}
where $\phi \in \set{B}_{\set{S} \times \set{A}}, \| \phi \| < \varepsilon$. By
\begin{align*}
    &\op{m}_\beta \left[ (k - 1) Q^* + q_0 \right] \\
    &\hspace{10mm}= (k-1) V^* + \op{m}_\beta \left[ (k-1) A^* + q_0 \right],
\end{align*}
we have
\begin{align*}
    Q_k = V^* + q_0 + k A^* - \op{m}_\beta \left( (k-1)A^* + q_0 \right) + \phi.
\end{align*}
This concludes the proof.
\end{proof}

\section{Proof of Theorem~\ref{theorem:bound_agvi}.}\label{sec:proof_bound_agvi}

\subsection{Proof Sketch}\label{subsec:proof_sketch_2}
To bound $\| Q^* - Q^{\pi_k} \|$, we first bound it by $\|Q^* - Q^\theta \| + \| Q^\theta - Q^{\pi_k} \|$. Proposition~\ref{proposition:distance_btw_q_m_and_q_star} gives us an upper bound of the first term.

To bound the second term, we note that since $\pi_k$ denote a policy which satisfies $\pi_k Q_k = \op{m}_\beta Q_k$,
\begin{align*}
    &\left\| Q^\theta - Q^{\pi_k} \right\|\nonumber\\
    &= \| Q^\theta - q_{k+1} + q_{k+1} - \op{T}^{\pi_k} Q^\theta + \op{T}^{\pi_k} Q^\theta - Q^{\pi_k} \|\\
    &\leq \left\| Q^\theta - q_{k+1} \right\| + \left\| q_{k+1} - \op{T}^{\pi_k} Q^\theta \right\| + \left\| \op{T}^{\pi_k} Q^\theta - Q^{\pi_k} \right\|\\
    &\leq \left\| Q^\theta - q_{k+1} \right\| + \left\| q_{k+1} - \op{T}^{\pi_k} Q^\theta \right\| + \gamma \left\| Q^\theta - Q^{\pi_k} \right\|\\
    &\leq \frac{1}{1-\gamma} \left( \left\| Q^\theta - q_{k+1} \right\| + \left\| q_{k+1} - \op{T}^{\pi_k} Q^\theta \right\| \right).
\end{align*}
Lemma~\ref{lemma:q_k_bound_total} gives us an upper bound of $\left\| Q^\theta - q_{k+1} \right\|$.

To bound $\left\| q_{k+1} - \op{T}^{\pi_k} Q^\theta \right\|$, first note that
\begin{align*}
    q_{k+1}
    &= \frac{A_{k}}{A_{k+1}} \op{T}^{\pi_k} q_k + \frac{\alpha^k}{A_{k+1}} \left( \op{T}^{\pi_k} q_0 + E_k \right)\\
    &= \op{T}^{\pi_k} \left( \frac{A_{k}}{A_{k+1}} q_k + \frac{\alpha^k}{A_{k+1}} q_0 \right) + \frac{\alpha^k E_k}{A_{k+1}}.
\end{align*}
Therefore,
\begin{align*}
    &\left\| q_{k+1} - \op{T}^{\pi_k} Q^\theta \right\|\nonumber\\
    &\leq \gamma \frac{A_k}{A_{k+1}} \left\| q_k - Q^\theta \right\|+ \gamma \frac{\alpha^k}{A_{k+1}} \left\| q_0 - Q^\theta \right\| + \frac{\alpha^k \left\| E_k \right\|}{A_{k+1}}.
\end{align*}
An upper bound of $\left\| Q^\theta - q_{k+1} \right\|$ is again given by Lemma~\ref{lemma:q_k_bound_total}.

Combining those bounds, we can bound $\|Q^* - Q^{\pi_k} \|$.

\subsection{Proofs}
\begin{proposition}\label{proposition:distance_btw_q_m_and_q_star}
The distance between $Q^\theta$ and $Q^*$ is bounded by
\begin{equation*}
    Q^* - Q^\theta \leq \frac{\gamma}{1-\gamma} \frac{1-\alpha}{\beta} \log |\set{A}|.
\end{equation*}
\end{proposition}

\begin{proof}
As we showed in the proof of Lemma~\ref{lemma:q_k_bound}, we have $\forall f \in \set{B}_{\set{S} \times \set{A}}, s \in \set{S}$, $\op{m} f (s) - \op{m}_\theta f (s) \leq \theta^{-1} \log |\set{A}|$. Therefore,
\begin{align*}
    Q^\theta
    &= \left( \op{T}_\theta \right)^{K-1} \left( r + \gamma P \op{m}_\theta Q^\theta \right)\\
    &\geq \left( \op{T}_\theta \right)^{K-1} \left( r + \gamma P \op{m} Q^\theta - \gamma \frac{\log |\set{A}| }{\theta} \right)\\
    &\geq \left( \op{T} \right)^{K} Q^{\op{m}_\beta} - \gamma \frac{\log |\set{A}| }{\theta} \sum_{i=0}^{K-1} \gamma^i.
\end{align*}
This holds for any $K$. By $K \rightarrow \infty$, the claim holds.
\end{proof}

We next show an upper bound of $\left\|Q^\theta - q_{k+1} \right\|$.
\begin{lemma}\label{lemma:q_k_bound_total}
For $k \geq 0$,
\begin{align*}
    \left\| Q^\theta - q_{k+1} \right\|
    \leq \frac{1}{A_{k+1}} \sum_{i=0}^{k} \gamma^i \alpha^{k-i} \left( C_0 + \|E_{k-i} \| \right),
\end{align*}
where $C_0$ is defined in Lemma~\ref{lemma:q_k_bound}.
\end{lemma}
\begin{proof}
We use induction. For $k=0$, the claim clearly holds from Lemma~\ref{lemma:q_k_bound}.

Next, suppose that up to $k-1$, the claim holds. From Lemma~\ref{lemma:q_k_bound}, we have
\begin{align*}
    &\left\|Q^\theta - q_{k+1} \right\| \nonumber\\
    &\leq \gamma \frac{A_k}{A_{k+1}} \left\| Q^\theta - q_k \right\| + \frac{\alpha^k}{A_{k+1}} \left( C_0 + \| E_k\| \right)\\
    &\leq \gamma \frac{1}{A_{k+1}} \sum_{i=0}^{k-1} \gamma^{i} \alpha^{k-1-i} \left( C_0 + \|E_{k-1-i} \| \right)\nonumber\\
    &\hspace{35mm} + \frac{\alpha^k}{A_{k+1}} \left( C_0 + \| E_k\| \right)\\
    &= \frac{1}{A_{k+1}} \sum_{i=0}^{k} \gamma^i \alpha^{k-i} \left( C_0 + \|E_{k-i} \| \right).
\end{align*}
Therefore, for any $k$, the claim holds.
\end{proof}

Now, we prove Theorem~\ref{theorem:bound_agvi}.
\begin{proof}[Proof of Theorem~\ref{theorem:bound_agvi}]
We just need to show an upper bound of $\| q_{k+1} - \op{T}^{\pi_k} Q^\theta\|$. As explained in Sect.~\ref{subsec:proof_sketch_2},
\begin{align*}
    &\left\| q_{k+1} - \op{T}^{\pi_k} Q^\theta \right\|\nonumber\\
    &\leq \gamma \frac{A_k}{A_{k+1}} \left\| q_k - Q^\theta \right\|+ \frac{\gamma \alpha^k}{A_{k+1}} \left\| q_0 - Q^\theta \right\| + \frac{\alpha^k \left\| E_k \right\|}{A_{k+1}}.
\end{align*}
By using Lemme~\ref{lemma:q_k_bound},
\begin{align*}
    &\left\| q_{k+1} - \op{T}^{\pi_k} Q^\theta \right\|\nonumber\\
    &\leq \gamma \frac{1}{A_{k+1}} \sum_{i=0}^{k-1} \gamma^i \alpha^{k-1-i} \left( C_0 + \|E_{k-1-i} \| \right)\nonumber\\
    &\hspace{30mm}+ \frac{\alpha^k}{A_{k+1}} \left( 2 \gamma V_{max} + \left\| E_k \right\| \right)\\
    &\leq \frac{1}{A_{k+1}} \sum_{i=0}^{k} \gamma^i \alpha^{k-i} \left( C_0 + \|E_{k-i} \| \right).
\end{align*}
Accordingly,
\begin{align*}
    \left\| Q^\theta - Q^{\pi_k} \right\|
    \leq \frac{2}{A_{k+1} (1-\gamma)} \sum_{i=0}^{k} \gamma^i \alpha^{k-i} \left( C_0 + \|E_{k-i} \| \right).
\end{align*}
By definition, $\alpha^{k-i} \left\| E_{k-i} \right\| = \left\| \sum_{j=0}^{k-i} \alpha^j \varepsilon_{k-i-j} \right\|$. As a result, we have
\begin{align*}
    \left\| Q^\theta - Q^{\pi_k} \right\|
    \leq \frac{2 }{A_{k+1} (1-\gamma)} \left( C_k + \mathcal{E}_k \right).
\end{align*}

In Proposition~\ref{proposition:distance_btw_q_m_and_q_star}, $Q^* - Q^\theta \leq \frac{\gamma}{1-\gamma} \frac{1-\alpha}{\beta} \log |\set{A}|$ is shown. As a result,
\begin{align*}
    &\left\| Q^* - Q^{\pi_k} \right\| \nonumber\\
    &\leq \frac{\gamma}{1-\gamma} \frac{1-\alpha}{\beta} \log |\set{A}| + \frac{2}{1-\gamma} \frac{1-\alpha}{1-\alpha^{k+1}} \left( C_k + \mathcal{E}_k \right).
\end{align*}
This concludes the proof.
\end{proof}

\section{Connection to a Policy Search Method}\label{sec:proof_derivation_of_gvi}
We briefly explain what we are going to show. Consider the following modified state value function.
\begin{align*}
    &V_{\widetilde{\pi}}^{\pi} (s)= V^\pi(s) \nonumber\\
    &- \E^\pi \left[ \sum_{t \geq 0} \gamma^t \left( \frac{1}{\eta} D(s_t; \pi, \widetilde{\pi}) - \frac{1}{\theta} H(s_t; \pi) \right) \mid s_0 = s \right],
\end{align*}
where $D(s; \pi, \widetilde{\pi})$ is KL divergence between $\pi(\cdot|s)$ and $\widetilde{\pi}(\cdot|s)$, and $H(s; \pi)$ is entropy of $\pi(\cdot|s)$. This modified state value function satisfies Bellman equation in the following sense:
\begin{align*}
    V_{\widetilde{\pi}}^{\pi} (s) = \op{T}^\pi V_{\widetilde{\pi}}^{\pi} (s) - \frac{1}{\eta} D(s; \pi, \widetilde{\pi}) + \frac{1}{\theta} H(s; \pi),
\end{align*}
where $\op{T}^\pi f (s) := \sum_a \pi(a|s) \left[ r (s, a) + \gamma P f (s, a) \right]$ for any $f \in \set{B}_{\set{S}}$. By defining a new operator by $\op{L}^{\pi} f (s) := \op{T}^\pi f (s) - \eta^{-1} D(s; \pi, \widetilde{\pi}) + \theta^{-1} H(s; \pi)$, the above Bellman equation can be written as $V_{\widetilde{\pi}}^{\pi} (s) = \op{L}^{\pi} V_{\widetilde{\pi}}^{\pi} (s)$.

As one can easily see, $\op{L}^{\pi}$ is contraction with modulus $\gamma$. Therefore, $V_{\widetilde{\pi}}^{\pi}$ is the unique fixed point. An operator $\op{L}$ defined by
\begin{equation}\label{eq:bellman_optimality_operator}
    \op{L} f (s) := \sup_\pi \left[ \op{T}^\pi f (s) - \frac{1}{\eta} D(s; \pi, \widetilde{\pi}) + \frac{1}{\theta} H(s; \pi) \right],
\end{equation}
which is analogous to Bellman optimality operator, is also a contraction with the unique fixed point $V_{\widetilde{\pi}}^{\circ}$. It turns out that the fixed point is optimal modified state value function, i.e., $V_{\widetilde{\pi}}^{\circ}$ which satisfies for any policy $\pi$, $V_{\widetilde{\pi}}^{\circ} \geq V_{\widetilde{\pi}}^{\pi}$. We show that any policy which verifies $\op{L}^\pi V_{\widetilde{\pi}}^{\circ} = V_{\widetilde{\pi}}^{\circ}$ is optimal. Finally, we derive an explicit form of $\pi^\circ$ showing the existence of such a policy.

\subsubsection{Expressions of an Optimal Regularized Policy and Optimal Modified State Value Function}
\begin{lemma}[Contraction Property of $\op{L}$]
$\op{L}$ is a contraction with modulus $\gamma$.
\end{lemma}
\begin{proof}
Suppose two functions $f, g \in \set{B}_{\set{S}}$. Without loss of generality, we can assume $\op{L} f (s) \geq \op{L} g(s)$.
\begin{align*}
    &\op{L} f (s) - \op{L} g(s)\nonumber\\
    &= \sup_\pi \left[ \op{T}^\pi f (s) - \frac{1}{\eta} D(s; \pi, \widetilde{\pi}) + \frac{1}{\theta} H(s; \pi) \right] \nonumber\\
    &- \sup_\pi \left[ \op{T}^\pi g (s) - \frac{1}{\eta} D(s; \pi, \widetilde{\pi}) + \frac{1}{\theta} H(s; \pi) \right]\\
    &\leq \sup_\pi \left( \op{T}^\pi f (s) - \op{T}^\pi g (s) \right)\\
    &\leq \gamma \left\| f - g \right\|.
\end{align*}
This holds for arbitrary $s$. Therefore, $\op{L}$ is a contraction with modulus $\gamma$. Note that it also follows that there exists the unique fixed point of $\op{L}$, which we denote $V_{\widetilde{\pi}}^\circ$.
\end{proof}

\begin{lemma}[Optimality of $V_{\widetilde{\pi}}^\circ$]
For any policy $\pi$, $V_{\widetilde{\pi}}^\circ \geq V_{\widetilde{\pi}}^\pi$.
\end{lemma}
\begin{proof}
This is immediate if both $\op{L}$ and $\op{L}^\pi$ are monotone, i.e., $\forall f, g \in \set{B}_{\set{S}}, f \geq g \Rightarrow \op{L} f \geq \op{L} g$. Indeed, $\op{L} V_{\widetilde{\pi}}^\pi \geq \op{L}^\pi V_{\widetilde{\pi}}^\pi = V_{\widetilde{\pi}}^\pi$ holds. Therefore, if it is the case, $V_{\widetilde{\pi}}^\circ = \limk \left( \op{L} \right)^k V_{\widetilde{\pi}}^\pi \geq V_{\widetilde{\pi}}^\pi$.

It is clear that $\op{L}^\pi$ is monotone since $\op{T}^\pi$ is monotone. To see that $\op{L}$ is monotone, suppose two functions $f, g \in \set{B}_{\set{S}}$ such that $f \geq g$. Let $\pi_g$ denote a policy satisfying
\begin{align*}
    \op{L} g (s)
    &= \sup_\pi \left[ \op{T}^\pi g (s) - \frac{1}{\eta} D(s; \pi, \widetilde{\pi}) + \frac{1}{\theta} H(s; \pi) \right]\\
    &= \op{T}^{\pi_g} g (s) - \frac{1}{\eta} D(s; \pi_g, \widetilde{\pi}) + \frac{1}{\theta} H(s; \pi_g).
\end{align*}
Such a policy exists as shown in Lemma~\ref{lemma:existence_of_sup_policy}. Since $\op{T}^{\pi_g}$ is monotone, it is obvious that 
\begin{align*}
    \op{L} g (s)
    &= \op{T}^{\pi_g} g (s) - \frac{1}{\eta} D(s; \pi_g, \widetilde{\pi}) + \frac{1}{\theta} H(s; \pi_g)\\
    &\leq \op{T}^{\pi_g} f (s) - \frac{1}{\eta} D(s; \pi_g, \widetilde{\pi}) + \frac{1}{\theta} H(s; \pi_g) \leq \op{L} f (s).
\end{align*}
This shows that $\op{L}$ is monotone.
\end{proof}

\begin{lemma}[A Sufficient Condition of an Optimal Policy]
Any policy which verifies $\op{L}^\pi V_{\widetilde{\pi}}^{\circ} = V_{\widetilde{\pi}}^{\circ}$ is optimal.
\end{lemma}
\begin{proof}
Since $\op{L}^\pi$ is monotone, and $\op{L}^\pi V_{\widetilde{\pi}}^{\circ} = V_{\widetilde{\pi}}^{\circ}$, we have $V_{\widetilde{\pi}}^{\pi} = \limk \left( \op{L}^\pi \right)^k V_{\widetilde{\pi}}^{\circ} = V_{\widetilde{\pi}}^{\circ}$. Hence, for any policy $\pi'$, $V_{\widetilde{\pi}}^{\pi} \geq V_{\widetilde{\pi}}^{\pi'}$ showing that $\pi$ is optimal.
\end{proof}

\begin{lemma}[Expression of a Policy $\pi_f$ s.t. $\op{L} f = \op{L}^{\pi_f} f$]\label{lemma:existence_of_sup_policy}
For any $f \in \set{B}_{\set{S}}$, there exists a policy $\pi_f$ s.t. $\op{L} f = \op{L}^{\pi_f} f$. Furthermore, such a policy has the following form:
\begin{align*}
    \pi_f(a|s) &= \frac{\widetilde{\pi}(a|s)^\alpha \exp \left( \beta \left[ r (s, a) + \gamma P f(s, a) \right] \right)}{\sum_{a'} \widetilde{\pi}(a'|s)^\alpha \exp \left( \beta \left[ r (s, a') + \gamma P f(s, a') \right] \right)}.
\end{align*}
where $\alpha = \dfrac{\theta}{\theta + \eta}$, and $\beta = \dfrac{\theta \eta}{\theta + \eta}$.
\end{lemma}

\begin{proof}
The supremization in $\op{L} f (s)$ is just supremization of a concave function w.r.t. $\pi (a_i | s)$ where a constraint ($\pi(a|s)$ must sum up to $1$) is an affine function. Therefore, the following is a necessary and sufficient condition for the optimality of the solution: for any action $a$,
\begin{align}
    &\pi_f (a|s) \nonumber\\
    &= \widetilde{\pi} (a|s)^\alpha \exp \left( \beta \left[ r(s, a) + \gamma P f (s, a) + \lambda_s \right] - 1 \right),\label{eq:pi_opt_temp}
\end{align}
where $\lambda_s$ is a Lagrange multiplier of $\sum_a \pi_f (a|s) = 1$. Other constraints that $\pi_f (a|s) \geq 0$ are clearly not active. By plugging $\pi_f (a|s)$ into $\log \sum_a \pi_f (a|s) = 0$, we get
\begin{align*}
    \beta \lambda_s = 1 - \log \sum_a \widetilde{\pi}(a|s)^\alpha \exp \beta \left[ r(s, a) + \gamma P f (s, a) \right].
\end{align*}
Substituting $\beta \lambda_s$ in (\ref{eq:pi_opt_temp}) with the above expression,
\begin{align*}
    \pi_f(a|s) &= \frac{\widetilde{\pi}(a|s)^\alpha \exp \left( \beta \left[ r (s, a) + \gamma P f(s, a) \right] \right)}{\sum_{a'} \widetilde{\pi}(a'|s)^\alpha \exp \left( \beta \left[ r (s, a') + \gamma P f(s, a') \right] \right)}.
\end{align*}
\end{proof}

Together with all these results, we can show Theorem~\ref{theorem:expression_of_optimal_policy}.
\begin{proof}[Proof of Theorem~\ref{theorem:expression_of_optimal_policy}]
The claim on the expression of $\pi^\circ$ is immediate from Lemma~\ref{lemma:existence_of_sup_policy}. To see that $V_{\widetilde{\pi}}^{\pi^\circ}(s)$ satisfies the claim, first note that
\begin{align*}
    \log \pi^\circ (a|s)
    &= \alpha \log \widetilde{\pi} (a|s) + \beta Q_{\widetilde{\pi}}^{\pi^\circ} (s, a) \nonumber\\
    &- \log \sum_{a'} \widetilde{\pi} (a'|s)^\alpha \exp \left( \beta Q_{\widetilde{\pi}}^{\pi^\circ} (s, a') \right).
\end{align*}
Plugging this expression into $D$ and $H$ in $\op{L}^{\pi^\circ} V_{\widetilde{\pi}}^{\pi^\circ}$, we get
\begin{align*}
    V_{\widetilde{\pi}}^{\pi^\circ} (s)
    = \frac{1}{\beta} \log \sum_a \widetilde{\pi}(a|s)^\alpha \exp \left( \beta Q_{\widetilde{\pi}}^{\pi^\circ} ( s, a) \right).
\end{align*}
\end{proof}

\end{document}